\documentclass{sig-alternate-no-boilerplate}
\usepackage{amsmath,amsfonts,amssymb}
\usepackage{times}
\usepackage{helvet}
\usepackage{courier}
\usepackage[skip=10pt]{caption}

\usepackage{mysetup,xspace,pgfplots}

\usepackage{graphicx,subcaption}

\newcommand{\bias}{\ensuremath{\mathtt{bias}}}
\newcommand{\error}{\ensuremath{\mathtt{error}}}

\newenvironment{noindentlist2}
 {\begin{list}{\labelitemi}{\leftmargin=0em \itemindent=1em \itemsep=0pt}}
 {\end{list}}

\begin{document}
%\title{Using Worker Quality Scores to Improve Stopping Rules}
\title{How Many Workers to Ask? Adaptive Exploration \\for Collecting High Quality Labels}
\author{Ittai Abraham
\thanks{VMware. Email: {\tt ittaia@gmail.com}.}
\and Omar Alonso
\and Vasilis Kandylas
\and Rajesh Patel
\and Steven Shelford
\and Aleksandrs Slivkins
\thanks{All: Microsoft.
    \{\tt omalonso, vakandyl, rajeshpa, steven.shelford, slivkins\}@microsoft.com.}
}

%\date{May 2014}

%\CopyrightYear{2016}
%\setcopyright{acmlicensed}
%\conferenceinfo{SIGIR '16,}{July 17 - 21, 2016, Pisa, Italy}
%\isbn{978-1-4503-4069-4/16/07}\acmPrice{\$15.00}
%\doi{http://dx.doi.org/10.1145/2911451.2911514}

\maketitle
\begin{abstract}
Crowdsourcing has been part of the IR toolbox as a cheap and fast mechanism to obtain labels for system development and evaluation. Successful deployment of crowdsourcing at scale involves adjusting many variables, a very important one being the number of workers needed per human intelligence task (HIT).
We consider the crowdsourcing task of learning the answer to simple multiple-choice HITs, which are representative of many relevance experiments. In order to provide statistically significant results, one often needs to ask multiple workers to answer the same HIT. A stopping rule is an algorithm that, given a HIT, decides for any given set of worker answers to stop and output an answer or iterate and ask one more worker. In contrast to other solutions that try to estimate worker performance and answer at the same time, our approach assumes the historical performance of a worker is known and tries to estimate the HIT difficulty and answer at the same time. The difficulty of the HIT decides how much weight to give to each worker's answer.
In this paper we investigate how to devise better stopping rules given workers' performance quality scores. We suggest adaptive exploration as a promising approach for scalable and automatic creation of ground truth. We conduct a data analysis on an industrial crowdsourcing platform, and use the observations from this analysis to design new stopping rules that use the workers' quality scores in a non-trivial manner. We then perform a number of experiments using real-world datasets and simulated data, showing that our algorithm performs better than other approaches.

\end{abstract}

%\category{H.3.0}{Information Store and Retrieval}{General}
%\category{I.1.2} {Computing Methodologies}{Algorithms}

%\terms{Design, Experimentation, Measurement }

% Alex: changed for brevity. Must use semi-colons by instructions.
\xhdr{Keywords:} Crowdsourcing; label quality; ground truth; assessments; adaptive algorithms; multi-armed bandits.

%\vspace{1cm}

%%%
\section{Introduction}

Crowdsourcing has become a central tool for improving the quality of search engines and many other large scale on-line services that require high quality assessments or labels. In this usage of crowdsourcing, a task or parts thereof are broadcast to multiple independent, relatively inexpensive workers, and their answers are aggregated.
Automation and optimization of this process at a large scale allows to significantly reduce the costs associated with setting up, running, and analyzing experiments that contain such tasks.

In a typical industrial scenario that we consider in this paper, a \emph{requester} has a collection of HITs, where each HIT has a specific, simple structure and involves only a small amount of work. We focus on multiple-choice HITs, that is, a HIT that contains  a question with several possible answers. The goal of the requester is to learn the preference of the crowd on each of the HITs. For example, if a HIT asks whether a particular URL should be labeled as spam and most workers believe it should, then the requester would like to learn this. This abstract scenario with multiple-choice HITs covers important industrial applications such as relevance assessment and other optimizations for a web search engine and construction of training sets for machine learning algorithms. Obtaining high quality labels is not only important for both model training and development but also for quality evaluation.

%A crowdsourcing system allows to quickly and cheaply obtain and aggregate information from a large on-line community of people. One of the most prominent uses of crowdsourcing  is in learning the \emph{wisdom of the crowd}.

%search engine optimization, relevance assessment, and training set constructions.

The requester has two goals: extract high-quality information from the crowd (i.e., reduce the error rate), and minimize costs (e.g., in terms of money and time spent). There is a tension between these two goals; we will refer to it as the \emph{quality-cost trade-off}.
In practice, it is assumed that there is some noise from the crowd, so the requester defines in advance how many workers are needed per assignment for the whole task. This approach may not
always be the right thing to do. For example, assessing the relevance of the query-URL pair (\texttt{\small facebook}, \texttt{\small www.facebook.com}) should need no more than one or two workers for such popular destination. In contrast, the pair
(\texttt{\small solar storms},  \texttt{\small solarstorms.org}) would require more workers as the topic may not be familiar to some. Using a fixed
number of workers may result in wasting resources for cases that are not needed or in not pooling more answers in assignments that
require more power.
Wouldn't it be useful if there is a flexible mechanism for adjusting the number of workers?

For cost-efficiency, one needs to take into account the heterogeneity in task difficulty and worker skills: some tasks are harder than others, and some workers are more skilled than others.
Further, workers' relative skill level may vary from one task to another.
% may be more skilled at some tasks and less skilled at others.
%\footnote{Also, workers' skill levels may vary over time,
% e.g. as a worker learns or becomes more/less motivated; we do not explicitly address this aspect.}
In general, it is desirable to
% route tasks to the right skilled workers and
(1) use less aggregation for easier tasks, (2) use more skilled workers. The crowdsourcing system initially has a very limited knowledge of task difficulty, and possibly also of worker skills, but both can, in principle, be learned over time.
%\\

A common application that stems from the assessment scenario is the generation of ground truth or gold standard, usually called \emph{gold} HITs or \emph{honey pots}.
These gold HITs are a set of HITs where the associated answers are known in advance. They can be a very effective mechanism to measure the performance of workers and data quality.
% Previous research and our own experience have shown that task designers  who use  gold HITs generally get high quality data on the experiments.
Gold HITs are usually generated manually, typically by hired domain experts. This approach is not scalable: it is expensive, time consuming and error prone. We believe that much more automated systems should be available, whereby a requester starts with a relatively small gold HIT set for bootstrapping, and uses the crowd to generate arbitrarily larger gold HIT sets of high quality. A central challenge in designing a mechanism for automated gold HIT creation is cost-efficient quality control. With error-prone workers, one needs to aggregate the answers of several workers to obtain a statistically robust answer for a gold HIT.

%\ascomment{This entire para seems redundant to me.}
%itt: i prefer to keep it
% AS: and moreover, it now clashes with the index-based algo
\OMIT{ %%%%%
We provide a data analysis of a real crowdsourcing system and design new stopping rule algorithms based on our conclusions. There could be other ways a quality score can be used, but we argue that there is good reason to consider systems that only use them to optimize the stopping rule. In particular, we assume that each worker gets paid the same payment  for each HIT she answers and we assume that the platform cannot choose which workers show up and answer the tasks. These assumptions match the current behavior of many industrial crowdsourcing platforms. Moreover adding differential payments and/or differential task assignment introduces a wide range of complexities. We believe that modifying these assumptions may raise some interesting research questions that are beyond the scope of this paper.
} %%%%%%%

\OMIT{ %%%%%%% not sure what's the point of this para
Gathering relevance labels at scale is a difficult problem that requires an iterative and time consuming process. Our goal is to help the requester by providing an algorithmic solution for collecting high quality labels for any type of multiple-choice tasks.
} %%%%%%%%

%itt: changed wording and make stronger
We make the following contributions: (1) data analysis of HITs from a production  platform, (2) design of two new stopping rule algorithms and (3) automatic generation of ground truth at scale.  We now describe the specifics insights and improvements.

%We break down our contributions in two main areas: data analysis and algorithms. In each, we describe the specifics insights and improvements.

\xhdr{1. Data analysis of a crowdsourcing platform .}
We collected and analyzed a real-world data set from logs of UHRS, a large in-house crowdsourcing platform operated by a comercial search engine. We note that our data set cannot be easily replicated on a publicly accessible crowdsourcing platform such as Amazon Mechanical Turk. Indeed, this is a much larger data set (250,000 total answers) than one could realistically collect via targeted experiments (i.e., without access to platform's logs) because of budget and time limitations. Moreover, using realistic HITs in an open experiment tends to be difficult because of trade secrets.

Analysing the data, we make two empirical observations. First, we find that the difficulty of a random HIT is distributed near-uniformly across a wide range. Second, we investigate the interplay between HIT difficulty and worker quality, and we find that the high-quality workers are significantly better than the low-quality workers for the relatively harder tasks, whereas there is very little difference between all workers for the relatively easy tasks. These observations motivate our algorithms and allow us to construct realistic simulated workloads.

The above observations are based on a large-scale data analysis, which makes them valuable even if they may seem intuitive to one's common sense (albeit perhaps counterintuitive to someone else's).
UHRS, Amazon Mechanical Turk, CrowdFlower, and others have similar architectural characteristics (e.g., HITs, task templates, payment system, etc.) so our data should be comparable to other platforms. Due to the proprietary nature of UHRS, this is the best information that we can share with the community.

\xhdr{2. Adaptive stopping rule algorithms.}
We consider obtaining a high-quality answer for a single HIT. We investigate a natural \emph{adaptive} approach in which the platform adaptively decides how many workers to use before stopping and choosing the answer. The core algorithmic question here is how to design a \emph{stopping rule}: an algorithm that at each round decides whether to stop or to ask one more worker. An obvious quality-cost trade-off is that using more workers naturally increases both costs and quality.
%Since HIT difficulty is near-uniformly distributed across a rather large range,
In view of our empirical observation, we do not optimize for a particular difficulty level, but instead design \emph{robust} algorithms that provide a competitive cost-quality trade-off for the entire range of difficulty.

As a baseline, we consider a scenario where workers are ``anonymous'', in the sense that the stopping rule cannot tell them apart. We design and analyze a simple stopping rule algorithm for this scenario, and optimize its parameters.

%they have no history and therefore no quality scores.

As workers vary in skill and expertise, one can assign quality scores to workers based on their past performance (typically, as measured on gold HITs). We investigate how these quality scores can help in building better stopping rules. While an obvious approach is to assign a fixed ``voting weight'' to each worker depending on the quality score, we find that more nuanced approaches perform even better. Given our empirical observations, we would like to utilize all workers for easy tasks, while giving more weight to better workers on harder tasks. As the task difficulty is not known a priori, we use the stopping time as a proxy: we start out believing that the task is easy, and change the belief in the ``harder'' direction over time as we ask more workers. We design a new adaptive strogging rule algorithm optimized for this setting. We conduct simulations based on the real workload, and conclude that this approach performs better than the ``fixed-weight'' approach.

We focus on the workers' quality scores that are given externally. This is for a practical reason: it is extremely difficult to design the entire crowdsourcing platform as a single algorithm that controls everything. Instead, one is typically forced to design the system in a modular way. In particular, while different requesters may want to have their own stopping rules, the crowdsourcing system may have a separate module that maintains workers' quality scores over different requesters.

\OMIT{ %%%%%%
Accordingly, we start with all workers having equal weights, and gradually modify the weights over time: increase weights for better workers and/or decrease weights for worse workers. We consider several algorithms based on this weight-modifying approach, and compare them to more obvious algorithms that do not change weights over time.
} %%%%%%%%

\xhdr{3. Scalable gold HIT creation.}
Creating gold HITs presents additional challenges compared to the normal HITs. As the quality of the entire application (or
successful experiment) hinges on the correctness of gold HITs, it is feasible and in fact desirable to route gold HITs to more reliable workers on the crowdsourcing platform.
%Note that the other workers would not be starved as they can work on other HITs. However,
Worker quality is typically estimated via performance on the gold HITs that are already present in the system, so the estimates may be very imprecise initially, and gradually improve over time as more gold HITs are added. To find answers for individual HITs in a cost-efficient manner, one can use stopping rules as described above.

We tackle these challenges using ideas from \emph{multi-armed bandits}, a problem space focused on sequentially choosing between a fixed and known set of alternatives with a goal to increase the cumulative reward and/or converge on the best alternative. A multi-armed bandit algorithm needs to trade off  \emph{exploration}, trying out various alternatives in order to gather information probably at the expense of short-term gains, and \emph{exploitation}, choosing alternatives that perform well based on the information collected so far.

We consider a stylized model in which HITs arrive one by one, and the system sequentially assigns workers to a given HIT until it concludes that the answer is known with sufficient confidence. In particular, such system needs to ``explore'' the available workers in order to estimate their quality. We incorporate an insight from multi-armed bandits called \emph{adaptive exploration}: not only the exploitation decisions, but also the exploration schedule itself can be adapted to the data points collected so far (e.g., we can give up early on low-performing alternatives).
To implement adaptive exploration, we take a well-known approach from prior work on multi-armed bandits and tailor it to our setting, connecting it with the stopping rules described above.
Our algorithm performs significantly better than the baseline uniform assignment of workers.

\vspace{1mm}

%\xhdr{Data sanitization.}
For algorithm evaluation we used the UHRS dataset discussed above, and also two previously published data sets from~\cite{Snow08,Ipeirotis10}. Since the UHRS dataset is somewhat sensitive, we have been required to sanitize our results, and in particular we only evaluate on simulated data parameterized by the key properties of that dataset.
% (parameterized by the properties of the collected real-life data).
One advantage of using a simulated workload is that one can replicate our algorithm evaluation (after choosing some values for the first column in Table~\ref{tab:worker-HIT-groups-diff}). Also, we have been able to generate as much simulated data as needed for the experiments, whereas the available number of workers in the original data set was insufficient for some HITs.

%\xhdr{A note on scope.}
%We do not consider the issue of worker availability. While some workers may be more available than others, it is not clear how to incorporate that into algorithm evaluation with our data set. Likewise, handling ``spammers'' is outside of our scope.

% (However, spammer issues at UHRS are different from ones in Amazon Turk, because on UHRS requesters can manage dedicated pools of workers.)}

%We measure each solution via the trade-off between \emph{cost} and \emph{quality}. In particular, we would like to obtain high quality (typically no more that 5\% error relative to a panel of domain experts), while minimizing the costs, in terms of both time and money.
%\xhdr{Organization of the paper.}
The paper is organized as follows. Section~\ref{related_work} summarizes the related work on this area. We describe preliminary
background in Section~\ref{preliminaries}. We provide an analysis using data from an industrial crowdsourcing platform in
Section~\ref{data_analysis}.
The design of a stopping rule for anonymous workers and its evaluation are described in Section~\ref{sec:unweighted}. Similarly,
the case for non-anonymous workers is described in Section~\ref{sec:weighted}. The gold HIT creation method is described in Section~\ref{gold-hit}.
%In Section~\ref{discussion} we list a number of practical recommendations based on our findings.
Finally, conclusions and future work are outlined in Section~\ref{conclusions}.

%\vspace{-0.3cm}
\section{Related work}
\label{related_work}

%For general background on crowdsourcing and human computation, refer to Law and von Ahn~\cite{Law11}.

The use of crowdsourcing as a cheap, fast and reliable mechanism for gathering labels was demonstrated in the areas of natural language processing~\cite{Snow08}, machine translation~\cite{Callison-Burch09} and information retrieval~\cite{AlonsoM12} by
running HITs on Amazon Mechanical Turk or CrowdFlower and comparing the results against an existing ground truth.
While early publications have shown that majority voting is a reasonable approach to achieve good results, new strategies have emerged in the last few years. Jointly with that, several papers consider \emph{task allocation}, the problem of allocating tasks to workers. % (or vice versa).

% Below we discuss the most related papers, mainly pointing out the differences between them and us. Due to the sheer volume of related work, a more thorough review is out of our scope.

Oleson et al.~\cite{Oleson} propose to use the notion of \emph{programmatic gold}, a technique that employs manual spot checking and detection of bad work, in order to reduce the amount of manual work. Ground truth creation is a problem for
new evaluation campaigns when no gold standard is available. Blanco et al.~\cite{Blanco11} rely on manual creation of gold answers for monitoring worker quality in a semantic search task. Scholer et al.~\cite{Scholer11} study the feasibility of
using duplicate documents as ground truth in test collections.

% Omar, do you really want to say "limited"? Why risk pissing someone off if we don't have to?
%While limited, the approach reduces the amount of manual work.

% A.S. It is really not clear what this means, in the context of this section.
% "suggest to use methods they call LU and NLU based on the beta distribution."

% A.S. no need to quote them, it seems that the text above is clear enough.
%\footnote{\cite{Sheng08} say they assume that
% ``individual labeling quality is independent of the specific data point being labeled''}.

Sheng et al.~\cite{Sheng08} design an algorithm that adaptively decides how many labels to use on a given HIT based on the distribution of all previously encountered HITs. Crucially, they assume that all HITs have the same difficulty for a given worker. However, our empirical evidence shows that HITs have widely varying difficulty levels; our algorithms are tailored to deal with this heterogeneity. Also, they optimize the quality of an overall classification objective, rather than the error rate.

Other approaches use the EM algorithm to estimate the workers' accuracy and the final HIT result at the same time \cite{dawid79}. The work presented in \cite{Ipeirotis10} is another algorithm based on EM, with several improvements. EM-based solutions use information from all the HITs in the data set and assume that a worker is answering many (or all) of these HITs and with more or less similar performance across them. Our approach is to consider each HIT individually and without using information from previously answered HITs. Because of this, we do not need to make the assumption that all the HITs are of similar difficulty. Additionally, it is not necessary to have the same workers answer multiple HITs. In fact, each HIT could be answered by a completely new set of workers. In a later section of this paper we make the additional assumption that we have knowledge of the overall quality of the workers, but we still consider that HITs could have varying (and unknown) difficulties.

Vox Populi~\cite{Dekel09} is a data cleaning algorithm that prunes low quality workers with the goal of improving a training set. The technique uses the aggregate label as an approximate ground truth and eliminates the workers that tend to provide incorrect answers.

Karger et al.~\cite{KOS11} optimize task allocation given budgets. Unlike ours, their solution is non-adaptive, in the sense that the task allocation is not adapted to the answers received so far. Further, \cite{KOS11} assume known Bayesian prior on both tasks and judges, whereas we do not.

From a methodology perspective, CrowdSynth~\cite{Kamar12} focuses on addressing consensus tasks by leveraging supervised learning.

Parameswaran et al.~\cite{CrowdScreen-sigmod12} consider a setting similar to our stopping rules for HITs with two possible answers. Unlike us, they assume that all HITs have the same difficulty level, and that the (two-sided) error probabilities are known to the algorithm. They focus designing algorithms for computing an optimal stopping rule.

%Adding a crowdsourcing layer as part of a computation engine is a very recent line of research. An example is CrowdDB, a system for crowdsourcing which includes human computation for processing queries~\cite{Franklin11}. CrowdDB offers basic quality control features, but we expect adoption of more advanced techniques as those systems become more available within the community.

Settings similar to stopping rules for anonymous workers, but incomparable on a technical level, were considered in prior work, e.g. \cite{Bechhofer59}, \cite{Ramey79}, \cite{Bechhofer85}, \cite{Dagum-sicomp00}, \cite{Mnih-icml08}, \cite{BanditSurveys-colt13}.

%\xhdr{Scalable gold HIT creation.}
For scalable gold HIT creation, our model emphasizes explore-exploit trade-off, and as such is related to multi-armed bandits; see \cite{CesaBL-book,Bubeck-survey12} for background on bandits and \cite{Crowdsourcing-PositionPaper13} for a discussion of explore-exploit problems that arise in crowdsourcing markets. Our algorithm builds on a bandit algorithm from Auer et al.~\cite{bandits-ucb1}.

Ho et al.~\cite{Jenn-icml13}, Abraham et al.~\cite{BanditSurveys-colt13} and Chen et al.~\cite{Chen-icml13} consider models for adaptive task assignment with heterogeneity in task difficulty levels and worker skill that are technically different from ours. In~\cite{Jenn-icml13}, the interaction protocol is ``inverted'': workers arrive one by one, and the algorithm sequentially and adaptively assigns tasks to each worker before irrevocably moving on to the next one. The exploration schedule in \cite{Jenn-icml13} is non-adaptive, unlike ours, in the sense that it does not depend on the observations already collected. The solution in~\cite{BanditSurveys-colt13} focuses on a single HIT. The algorithm in~\cite{Chen-icml13} develops an approach based on Bayesian bandits that requires exact knowledge of the experimentation budget and the Bayesian priors.

\section{Preliminaries}
\label{preliminaries}

A HIT is a question with a set of $S$ of possible answers. For each HIT we assume that there exists one answer which is the correct answer (more on this below under ``probabilistic assumptions''). A requester has a collection of HITs, which we call a \emph{workload}. The goal of the requester is to learn what is the correct answer for each HIT. The requester has access to a crowdsourcing system.

We model a stylized crowdsourcing system that operates in rounds. In each round the crowdsourcing systems chooses one HIT from the workload and a worker arrives, receives the HIT, submits her answer and gets paid a fixed amount for her work. The crowdsourcing system needs to output an answer for each HIT in the workload. The algorithm can adaptively decide for each HIT how many workers to ask for answers.
We mostly focus on a single HIT. In each round, a worker arrives and submits an answer to this HIT. The algorithm needs to decide whether to stop (stopping rule) and if so, which answer to choose (selection rule).

There are two measures to be minimized in such an algorithm: (1) the \emph{error rate} for this workload (the percentage of HITs for which the algorithm outputs the wrong answer), and (2) the \emph{average cost} for this workload (the average cost per HIT paid to the workers by the algorithm). Formally this is a bi-criteria optimization problem. If all workers are paid equally, the average cost is simply the average number of rounds.

%\textbf{Probabilistic assumptions.}
\xhdr{Probabilistic Assumptions.}
\label{sec:probab-assumptions}
We model each worker's answer as a random variable over $S$, and assume that these random variables are mutually independent. We assume that the most probable answer is the same for each worker. For the purposes of this paper, the ``correct answer'' is just the most probable answer, and this is the answer that we strive to learn.%
%\footnote{While the most probable answer may actually be false, we do not attempt to learn which answer is \emph{really} correct. Besides, it is not clear how to learn this from workers' responses on a particular HIT.}
The difference between the probability of the most probable answer and the second most probably answer is called the \emph{bias} of a given worker on a given HIT. This quantity, averaged over all workers, is the \emph{bias} of a given HIT.
%For a given task we can define its bias simply as the average probability
%of the most probable answer minus the average probability
%of the second most probable answer (where the average is taken over all workers).
A large bias (close to 1) corresponds to our intuition that the HIT is very easy: the error rate is very small, whereas a small bias (close to 0) implies that the HIT is hard, in the sense that it is very difficult to distinguish between the two most probable options. Here, our notion of easy/hard HITs is objective (reflecting agreement with majority), rather than subjective (reflecting workers' sentiments). Hereafter we use the bias of a HIT as a measure of its hardness.
% Hereafter we define the \emph{hardness} of a task as a direct function of its bias.
In particular, we say that HIT A is \emph{harder} than HIT B if the bias of A is smaller than the bias of B.

%\textbf{Adaptive Exploration}
\OMIT{ %%%%%%
\subsection{Adaptive Exploration}
Any good solution should involve \emph{adaptive HIT assignment}, the assignment of HITs to workers which changes based on previous observations.
This raises a natural trade-off between \emph{exploration} (experimentation to learn more about worker skill and task difficulty) and \emph{exploitation} (making optimal decisions based on the experimentation results available so far). This trade-off occurs in many different scenarios, and is well-studied in Machine Learning and Operations Research (see~\cite{CesaBL-book,Bubeck-survey12,Gittins-book11} for background).

In many settings, the best algorithms for explore-exploit trade-off involve \emph{adaptive exploration}, where not only the exploitation decisions, but the exploration schedule itself is adapted to the previous observations. For example, once we are sufficiently confident that a given alternative is bad, we can give up on it early and focus our exploration budget on more promising alternatives.
} %%%%%%

\section{Data Analysis}
\label{data_analysis}

We performed our analysis using data from UHRS, a large in-house crowdsourcing platform operated by Microsoft. UHRS is used by many different internal groups  for evaluation, label collection, and machine learning applications. The tasks range from TREC-like evaluations to domain specific labeling and experimentation. In particular, UHRS is used to gather training and evaluation data for various aspects of the search engine.

Using the logs of UHRS, we collected a data set from a variety of tasks and workers. In that data set, we selected all tasks that contained at least 50 HITs, and all HITs with at least 50 answers. These HITs have been used for training and/or quality control, which explains an unusually large number of answers per HIT. This large number has been essential for our purposes. We considered all HITs in all these tasks. This gave us a data set containing 20 tasks, 3,000 workers, 2,700 HITs, and 250,000 total answers. For each HIT we computed the majority answer, which we considered as the ``correct'' answer. Details of the different types of HITs, design templates,  and other specific metrics are left out due to proprietary information.

%\subsection{Empirical Biases of HITs}
\xhdr{Empirical Biases of HITs.}
Workers' replies to a given HIT are, at first approximation, IID samples from some fixed distribution $\mD$. A crucial property of $\mD$ is the difference between the top two probabilities, which we call bias of this HIT; note that the bias completely defines $\mD$ if there are only two answers. Informally, larger bias corresponds to easier HIT.
We study the distribution over biases in our workload. For each HIT, we consider the empirical frequencies of answers, and define ``empirical bias'' as the difference between the top two frequencies. We plot the CDF for empirical biases in Figure~\ref{fig:CDF-bias}.

\begin{figure}[h]
	\centering
%	\begin{subfigure}[b]{0.5\textwidth}
%        	\input{chart_gap_distribution_twoOptions.tex}
%                \caption{all HITs with 2 asnwers (136)}
%                \label{fig:CDF-gap-a}
%        \end{subfigure}
        %$\qquad$
 %       \begin{subfigure}[b]{0.5\textwidth}
		%\tikzstyle{every pin}=[fill=white, draw=black, font=\footnotesize]
\begin{tikzpicture}
\begin{axis}[
    width=8.5cm,
	%title={Aggregate gap distribution},
	%ylabel=Gap,
	%legend pos=south east,
    xmin=0,
    xmax=1,
    ymin=0,
    ymax=1,
	xticklabel={\pgfmathparse{\tick*100}\pgfmathprintnumber{\pgfmathresult}\%},
]
\addplot [color=blue, mark=.] table [x=percent, y=Gap] {chart_gap_distribution_manyOptions.data};
\addplot [color=black] coordinates { (0,0.1847) (0.8624,1) }
[yshift=-25pt] node[pos=0.5] {$R^2 = 0.9402$};
%\addlegendentry{gap}
\end{axis}
\end{tikzpicture}
%              \caption{all HITs, variable \#answers}
%             \label{fig:CDF-gap-b}
%    \end{subfigure}
        \caption{CDF for the empirical bias of HITs.}
        \label{fig:CDF-bias}
\end{figure}
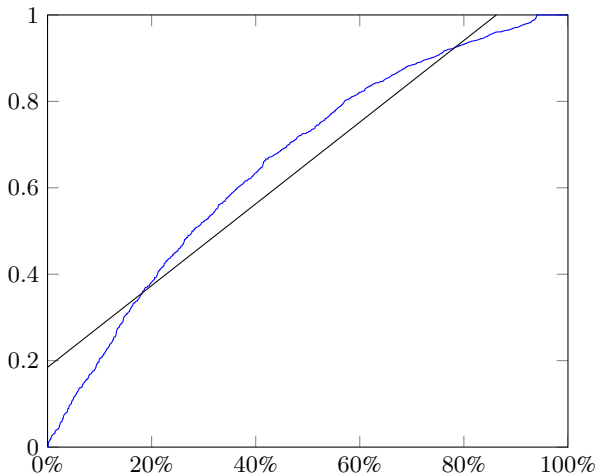

We conclude that HITs have a wide range of biases: some are significantly more difficult than others. In particular, tailoring a decision rule to HITs with a specific narrow range of biases is impractical. Further, we observe that the empirical distribution is, roughly, near-uniform. We use this observation to generate the simulated workload in the next section.

%\subsection{Error Rates}
\xhdr{Error Rates.}
For each worker, we compute the average error rate across all HITs that she answered. According to that, we split all workers into 9 equally sized groups, from best-performing ($W_0$) to worst-performing ($W_8$). Similarly, for each HIT we compute the average error rate across all workers that answered it. %According to that,
We split all HITs into 9~equally-sized groups, from easiest ($H_0$) to most difficult ($H_8$). Let $\error(W_i,H_j)$ be the average error rate of the workers in the worker group $W_i$ when answering the HITs in the HIT group $H_j$.

To make our main finding clearer, and also because our data set is somewhat sensitive, we report a 9-by-8 table (see Table~\ref{tab:worker-HIT-groups-diff}): for each HIT group $H_i$, $i=0\ldots 8$ and each worker group $W_j$, $j = 1\ldots 8$, the corresponding cell contains the difference
\begin{align}
\error(W_i,H_j)-\error(W_0,H_j).
\label{eq:table-cell}
\end{align}
The table is also visualized as a heat map in Figure \ref{fig:heatmap}.

\begin{table}[h]
\begin{center}
\begin{tabular}{c|ccccccccc}
\% & $W_1$&     $W_2$ &     $W_3$ &    $W_4$ &     $W_5$ &     $W_6$ &     $W_7$ &     $W_8$ \\ \hline
$H_0$& 0 &      0 &      0 &      1 &      1 &      1 &      2 &      4 \\
$H_1$& 1 &      1 &      2 &      2 &      3 &      4 &      6 &     15 \\
$H_2$& 1 &      3 &      3 &      4 &      6 &      8 &     11 &     20 \\
$H_3$& 1 &      4 &      4 &      7 &      7 &     11 &     16 &     27 \\
$H_4$& 4 &      7 &      8 &     12 &    13 &     17 &     23 &     36 \\
$H_5$& 5 &      9 &     11 &    14 &     18 &     20 &     26 &     43 \\
$H_6$& 7 &     11 &    15 &    18 &     22 &     25 &     30 &     47 \\
$H_7$& 11 &   14 &    19 &    21 &     25 &     26 &     33 &     48 \\
$H_8$& 19 &   24 &    27 &    29 &     31 &     35 &     39 &     50
\end{tabular}
\end{center}
\caption{Error rates for different worker/HIT groups. The cell $(W_i,H_j)$ contains the difference \eqref{eq:table-cell}, in percent points.}
\label{tab:worker-HIT-groups-diff}
\end{table}

\begin{figure}[h]
	\includegraphics[width=8cm]{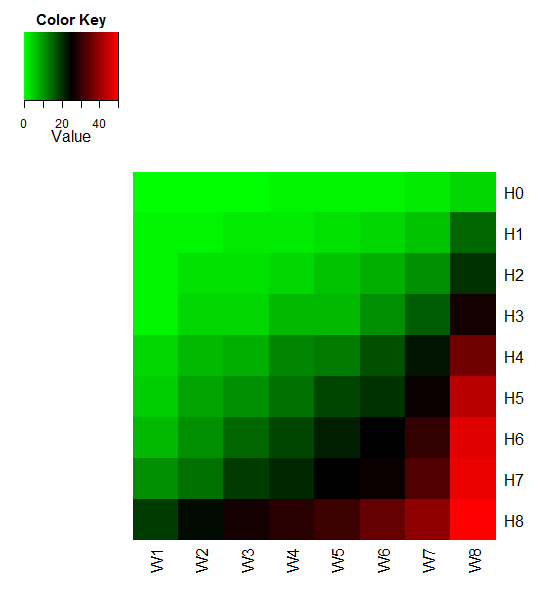}
	\caption{Error rates for different worker/HIT groups.}
	\label{fig:heatmap}
\end{figure}

%\vspace{0.5cm}

%\subsection{Findings}
\xhdr{Findings.}
From Table~\ref{tab:worker-HIT-groups-diff}, we make the following observations.
For difficult tasks ($H_6 \ldots H_8$) the set of good judges ($W_0\ldots W_2$) is significantly better (has a lower error rate) than the set of bad judges ($W_6\ldots W_8$).
For easy tasks ($H_0 \ldots H_2$) there is very little difference between \emph{all judges} (expect perhaps for the very worst judges).
%\begin{itemize}
%\item One is expected: For difficult tasks ($H_6 \ldots H_8$) the set of good judges ($W_0\ldots W_2$) is significantly better (has a lower error rate) than the set of bad judges ($W_6\ldots W_8$).
%
%\item One is somewhat surprising: For easy tasks ($H_0 \ldots H_2$) there is very little difference between \emph{all judges} (expect perhaps for the very worst judges).
%\end{itemize}

These observations are robust to changing the number of HIT and worker groups (from 5 to 9). To summarize, {\em the difference in performance between good and bad workers is much more significant for harder HITs than for easier HITs.} Accordingly, we devise algorithms that tend to use all workers for easier HITs, and favor better workers for more difficult HITs.

\section{Stopping rule for \\anonymous workers}
\label{sec:unweighted}

We start with a simpler case when workers are anonymous, in the sense that there is no prior information on which workers are better than others. Absent such information, we treat all workers equally: essentially, we give each worker's vote the same weight.

\OMIT{
% Alex to Ittai: I think we do not need to state this as an assumption / intuition.
%                 What we say in "discussion" suffices. Pls call me if you disagree.
Further, we rely on an intuition that the requester is willing to tolerate a higher error rate for HITs with a very small bias, in order to improve the error rate vs. average cost trade-off for the entire workload.
}

%\xhdr{Algorithm for anonymous workers.}
\subsection{Algorithm}
For simplicity, let us assume there are only two answers for a HIT: $A$ and $B$. In each round $t$, let $V_{A,t}$ and $V_{B,t}$ be the number of workers that vote for $A$ and $B$, respectively. Note that $t=V_{A,t}+V_{B_t}$. Our stopping rule is as follows:
\begin{align}\label{eq:stopping-rule-unweighted}
\text{Stop if}\;
|V_{A,t} - V_{B,t}| \geq C\sqrt{t} - \eps t.
\end{align}
Here $\eps \geq 0$ and $C \geq 0$ are parameters that need to be chosen in advance. After the algorithm stops, the selection rule is simply to select the most frequent answer.
Note that the right-hand side is not an integer, so we can randomly round it to one of the two closest integers in a way that is proportional to the fractional part.

%\subsubsection{Discussion}
\xhdr{Discussion.}
Our intuition is that each worker's reply is drawn IID from some fixed distribution over answers; recall that the bias of a HIT is the difference between the top two probabilities in this distribution. For two answers:
    $$\bias = |\Pr[A]-\Pr[B]|$$

Informally, the meaning of parameter $\eps$ is that we are willing to tolerate a higher error rate for HITs with $\bias \leq \eps$, in order to improve the error-cost trade-off for the entire workload. We find in our simulations that a small value of $\eps$ performs better than $\eps=0$.

Parameter $C$ controls the error-cost trade-off: increasing it increases the average cost and decreases the error rate.
In practice, the parameters $(C,\eps)$ should be adjusted to typical workloads to obtain the desirable error-cost trade-off.

%\subsubsection{Analysis}
\xhdr{Analysis.}
For the sake of analysis, let us consider a slight modification of algorithm~\eqref{eq:stopping-rule-unweighted} in which parameter $C$ is proportional to $\log t$ (we view this dependence as minor compared to the $\sqrt{t}$ term). 

We prove that our algorithm returns a correct answer with high probability if $\bias \geq \eps$. We consider two hypotheses:
\begin{description}
\item[(H1)] The correct answer is A and $\bias \geq \epsilon$,
\item[(H2)] The correct answer is B and $\bias \geq \epsilon$.
\end{description}
Effectively, if one hypothesis is right, our algorithm rejects the other with high probability.

With $\eps=0$, the expected cost (stopping time) is on the order of $\bias^{-2}$, in line with standard results on biased coin tossing. Using $\eps>0$ relaxes this to $(\eps+\bias)^{-2}$.

\begin{lemma}\label{lm:error-rate}
Fix $\delta\in(0,1)$. Consider the algorithm~\eqref{eq:stopping-rule-unweighted} with parameters $\eps>0$ and
    $C=C_t =  \sqrt{\log (t^2/\delta)}$.
Suppose this algorithm is applied to a HIT with $\bias = \eps_0$.
\begin{description}
\item[(a)] If $\eps_0\geq \eps$ then the algorithm returns a correct answer with probability at least $1-O(\delta)$.
\item[(b)] The expected cost (stopping time) is at most
    $O\left( \rho^{-2}\, \log \tfrac{1}{\delta \rho} \right)$,
where $\rho = \eps+\eps_0$.
\end{description}
\end{lemma}

\begin{proof}
W.l.o.g., suppose $\Pr[A]\geq \Pr[B]$. Consider the difference $Z_t = V_{A,t} - V_{B,t}-\eps_0 t$, where $t$ ranges over rounds. The increments $Z_t-Z_{t-1}$ are independent random variables with mean $0$ and values $\pm 1$, so $Z_t$ is a random walk. Therefore for each $t$:
\begin{align}\label{eq:lm:error-rate-1}
\Pr[ |Z_t| \leq C_t\sqrt{t}] \geq 1-O(\delta/t^2).
\end{align}
by a standard application of the \emph{Azuma-Hoeffding Inequality}. Taking the Union Bound over all $t$, it follows that
\begin{align}\label{eq:lm:error-rate-2}
\Pr\left[ |Z_t| \leq C_t\sqrt{t} \quad\text{for all $t$} \right] \geq 1-O(\delta).
\end{align}

For part (a) assume that hypothesis (H1) holds, i.e. that $\eps_0\geq \eps$, but the algorithm returns an incorrect answer, i.e. stops at some round $t$ so that answer $B$ is chosen. We show this cannot happen if the high-probability event in \eqref{eq:lm:error-rate-2} holds. Indeed, at such round $t$:
\begin{align*}
V_{B,t} - V_{A,t} &> C_t\sqrt{t} - \eps t \\
Z_t  &< (\eps-\eps_0) t-C_t\sqrt{t} < -C_t\sqrt{t}.
\end{align*}
The latter contradicts the high-probability event in \eqref{eq:lm:error-rate-2}.

For part (b), let $T$ be the stopping time. Consider round $t$ such that
$t\geq (2C_t/\rho)^2$. For any such round, the high-probability event $\{Z_t>-C_t \sqrt{t}\}$ implies that
\begin{align*}%\label{eq:lm:error-rate-3}
    V_{A,t} - V_{B,t} > -C_t\sqrt{t}+\eps_0 t \geq C_t\sqrt{t} - \eps t,
\end{align*}
so the algorithm stops at round $t$ or earlier, i.e., $T\leq t$.
By \eqref{eq:lm:error-rate-1}, we conclude that
    $\Pr[T>t] <O(\delta/t^2)$. 
Now, there exists $t_0 = O(\rho^{-2} \log{\frac{1}{\delta \rho}})$ such that 
    $t\geq (2C_t/\rho)^2$
for all $t\geq t_0$. Therefore:
\begin{align*}
\textstyle
\E[T] = \sum_{t=1}^\infty \Pr[T>t] 
    \leq t_0 + \sum_{t>t_0} \delta/t^2
    = t_0 + O(\delta).
\end{align*}
So the expected stopping time $\E[T]$ is as small as we claimed.
\end{proof}

\OMIT{Parameter $\eps$ represents an estimated guarantee on the bias (drift) of the opposite hypothesis that we are rejecting. In a way, if the bias (drift) is smaller than $\eps$ then this allows an incorrect answer that will increase the error rate. }

\newcommand{\curve}{varying-$C$ curve\xspace}
\newcommand{\curves}{varying-$C$ curves\xspace}

%\subsubsection{Extension to multiple answers}
\xhdr{Extension to multiple answers.}
One can extend the stopping rule \eqref{eq:stopping-rule-unweighted} to more than two answers in an obvious way. At time $t$, let $A^*(t)$ and $B^*(t)$ be the answers with the largest and second-largest number of votes, respectively. The stopping rule is
\begin{align}\label{eq:stopping-rule-unweighted-multiple}
\text{Stop if}\;
V_{A^*(t),t} - V_{B^*(t),t} \geq  C\sqrt{t} - \eps t.
\end{align}
The selection rule is to select the most frequent answer.

Lemma~\ref{lm:error-rate} easily carries over to multiple answers. (The proof considers a separate random walk for each pair of answers $A,B$:
$$
    Z_t^{A,B} = V_{A,t} - V_{B,t}- t (\Pr[A]-\Pr[B]),
$$
obtains high-probability event 
    $\{|Z^{A,B}_t| \leq C_t\sqrt{t} \}$
as in \eqref{eq:lm:error-rate-1}, and then conditions on the intersection of all such events.)

%\xhdr{Experimental results.}
\subsection{Experimental Results}
\label{expResultsAnonymous}
We used a simulated workload, consisting of 100,000 HITs, each with two answers. For each HIT, the bias towards the correct answer (the difference between the probabilities of the two answers) was chosen uniformly at random in the interval $[0.1, 0.6]$. This closely matches an empirical distribution of biases, as we have found in the previous experiments. For each worker answering this HIT, the answer was chosen independently at random with the corresponding bias.

For each pair $(\eps,C)$ of parameters, running our algorithm on a single HIT gives a two-fold outcome: the cost and whether the correct answer was chosen. Thus, running our algorithm on all HITs in our workload results in two numbers: average cost and error rate (over all HITs). We plot these pairs of numbers on a coordinate plane where the axes are average cost and error rate. Thus, fixing $\eps$ and varying $C$ we obtain a curve on this plane, which we call the \emph{\curve}.

\begin{figure}[h]
\centering
\begin{tikzpicture}
\begin{axis}[
        width=9cm,
        height=8cm,
        xlabel=Error rate,
        ylabel=Average cost,
        every axis y label/.style={at={(ticklabel cs:0.5)}, rotate=90, anchor=center, yshift=1.5mm},
        legend pos=north east,
        legend cell align=left,
        xmin=0,%0.04,
        xmax=0.17,
        ymin=0,%5,
        xticklabel style={
                /pgf/number format/.cd,
                      fixed,
                fixed zerofill,
                precision=2,
        },
        cycle list name=color list
]
%\addlegendimage{empty legend}
%\addlegendentry{\eps}
\addplot [color=green,dashed,thick] table [x=fixed_overlap_error_rate, y=fixed_overlap_avg_cost] {chart1b.data};
\addlegendentry{fixed overlap majority}
\addplot [color=red,solid] table [x=adap_error_rate_0, y=adap_avg_cost_0] {chart1b.data};
\addlegendentry{adaptive \eps = 0}
\addplot [color=blue,dashed] table [x=adap_error_rate_0.1, y=adap_avg_cost_0.1] {chart1b.data};
\addlegendentry{adaptive \eps = 0.1}
\addplot [color=gray,dotted,thick] table [x=adap_error_rate_0.2, y=adap_avg_cost_0.2] {chart1b.data};
\addlegendentry{adaptive \eps = 0.2}
\addplot [color=orange,densely dotted] table [x=adap_error_rate_0.3, y=adap_avg_cost_0.3] {chart1b.data};
\addlegendentry{adaptive \eps = 0.3}
\addplot [color=magenta,dashdotted] table [x=adap_error_rate_0.5, y=adap_avg_cost_0.5] {chart1b.data};
\addlegendentry{adaptive \eps = 0.5}
\addplot [color=black,loosely dotted,thick] table [x=adap_error_rate_1, y=adap_avg_cost_1] {chart1b.data};
\addlegendentry{adaptive \eps = 1}
\end{axis}
\end{tikzpicture}
\caption{Cost-quality trade-off for fixed overlap majority and the adaptive algorithm on a simulated data set.}
\label{fig:chart1b}
\end{figure}
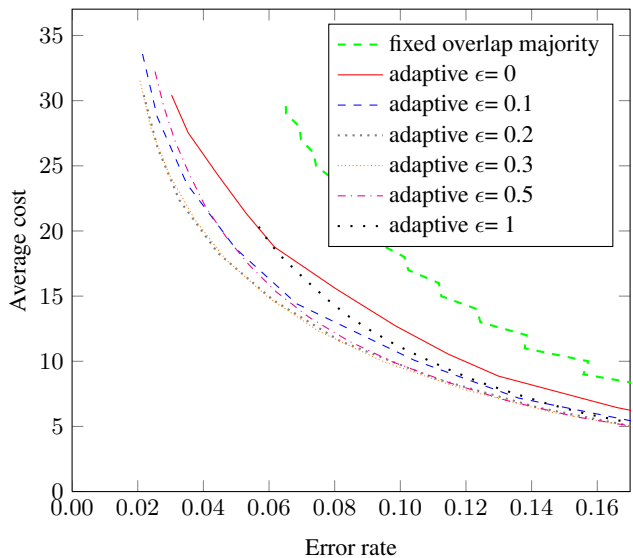

We consider several values for $\eps$, ranging from $0$ to $1$. For each value of $\eps$, we plot the corresponding \curve (Figure~\ref{fig:chart1b}). We also plot, as baseline, the fixed overlap majority algorithm, which uses a fixed number of annotations per HIT (we vary from 1 to 30) and uses simple majority voting (breaking ties randomly). This technique is used in \cite{Snow08}. Surprisingly, we find that, up to some minor noise, for any two \curves it holds that one lies below another. This did not have to be the case, as two curves could criss-cross. If one \curve lies below another \curve, this means that the $\eps$ parameter for the former curve is always better: for any $C$, it gives better average cost for the same error rate. Thus, we find that for any two significantly different values of parameter $\eps$, one value is better than another, regardless of the $C$. From Figure~\ref{fig:chart1b}, we find that the most promising range for $\eps$ is $[.2, .3]$. We have omitted the less interesting values for clarity.
%We zoom in on this range in Figure~\ref{fig:chart2}.

\begin{figure}[h]
\centering
\begin{tikzpicture}
\begin{axis}[
        width=9cm,
        height=8cm,
        xlabel=Error rate,
        ylabel=Average cost,
        every axis y label/.style={at={(ticklabel cs:0.5)}, rotate=90, anchor=center, yshift=1.5mm},
        legend pos=north east,
        legend cell align=left,
        xmin=0,%0.05,
        xmax=0.3,
        ymin=0,%1,
        xticklabel style={
                /pgf/number format/.cd,
                      fixed,
                fixed zerofill,
                precision=2,
        },
        cycle list name=color list
]
%\addlegendimage{empty legend}
%\addlegendentry{\eps}
\addplot [color=green,dashed,thick] table [x=fixed_overlap_error_rate, y=fixed_overlap_avg_cost] {chart1b_RTE.data};
\addlegendentry{fixed overlap majority}
\addplot [color=red,solid] table [x=adap_error_rate_0, y=adap_avg_cost_0] {chart1b_RTE.data};
\addlegendentry{adaptive \eps = 0}
\addplot [color=blue,dashed] table [x=adap_error_rate_0.1, y=adap_avg_cost_0.1] {chart1b_RTE.data};
\addlegendentry{adaptive \eps = 0.1}
\addplot [color=gray,dotted,thick] table [x=adap_error_rate_0.2, y=adap_avg_cost_0.2] {chart1b_RTE.data};
\addlegendentry{adaptive \eps = 0.2}
\addplot [color=orange,densely dotted] table [x=adap_error_rate_0.3, y=adap_avg_cost_0.3] {chart1b_RTE.data};
\addlegendentry{adaptive \eps = 0.3}
\addplot [color=magenta,dashdotted] table [x=adap_error_rate_0.5, y=adap_avg_cost_0.5] {chart1b_RTE.data};
\addlegendentry{adaptive \eps = 0.5}
\addplot [color=black,loosely dotted,thick] table [x=adap_error_rate_1, y=adap_avg_cost_1] {chart1b_RTE.data};
\addlegendentry{adaptive \eps = 1}
\end{axis}
\end{tikzpicture}
\caption{Cost-quality trade-off for fixed overlap majority and the adaptive algorithm on the RTE data set.}
\label{fig:chart1b_RTE}
\end{figure}
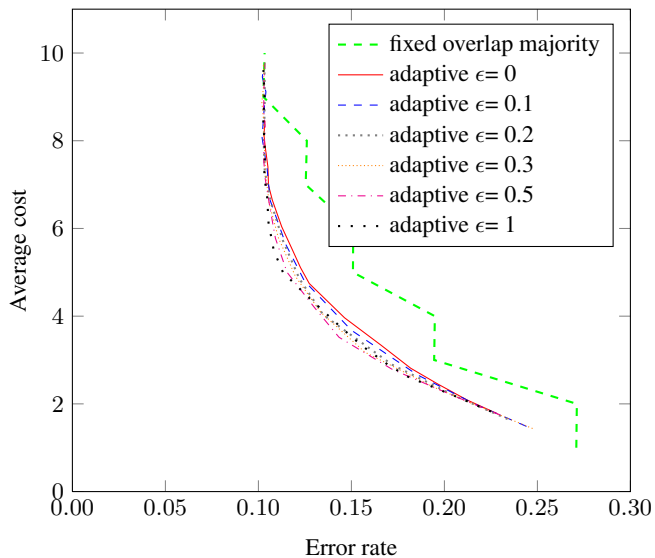

We repeat the same experiment with the NLP RTE data set from \cite{Snow08}. The RTE data set contains 800 HITs and 10 annotations per HIT.
%The experiments were repeated 100 times.
For the fixed overlap majority algorithm we randomly sampled a fixed number of worker labels per HIT (varying the fixed number to produce the curve), whereas for the other algorithms we randomly permuted the order of the worker labels. The results we report are the averages of the error rates and costs over  100 runs (Figure~\ref{fig:chart1b_RTE}). As with the simulated data set, the adaptive algorithm performs better than the fixed overlap majority. Because of the 10 annotations per HIT, it is impossible to do better than the minimum error rate of about 0.11, which is achieved when all available annotations are used. However, the adaptive algorithm can achieve approximately the same error with much lower average cost (about~6 instead of~10).

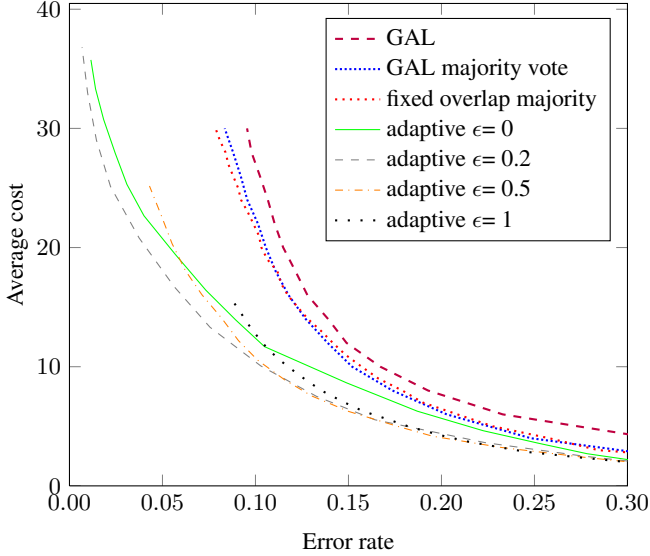
\begin{figure}[h]
\centering
\begin{tikzpicture}
\begin{axis}[
        width=9cm,
        height=8cm,
        xlabel=Error rate,
        ylabel=Average cost,
        every axis y label/.style={at={(ticklabel cs:0.5)}, rotate=90, anchor=center, yshift=1.5mm},
        legend pos=north east,
        legend cell align=left,
        xmin=0,%0.04,
        xmax=0.3,
        ymin=0,%1,
        xticklabel style={
                /pgf/number format/.cd,
                      fixed,
                fixed zerofill,
                precision=2,
        },
        cycle list name=color list
]
%\addlegendimage{empty legend}
%\addlegendentry{\eps}
\addplot [color=purple,dashed,thick] table [x=gal_error_rate, y=gal_avg_cost] {chart1b_EM.data};
\addlegendentry{GAL}
\addplot [color=blue,densely dotted,thick] table [x=gal_pre_error_rate, y=gal_pre_avg_cost] {chart1b_EM.data};
\addlegendentry{GAL majority vote}
\addplot [color=red,dotted, thick] table [x=fixed_overlap_error_rate, y=fixed_overlap_avg_cost] {chart1b_EM.data};
\addlegendentry{fixed overlap majority}
\addplot [color=green,solid] table [x=adap_error_rate_0, y=adap_avg_cost_0] {chart1b_EM.data};
\addlegendentry{adaptive \eps = 0}
%\addplot [color=blue,dashed] table [x=adap_error_rate_0.1, y=adap_avg_cost_0.1] {chart1b_EM.data};
%\addlegendentry{adaptive \eps = 0.1}
\addplot [color=gray,dashed] table [x=adap_error_rate_0.2, y=adap_avg_cost_0.2] {chart1b_EM.data};
\addlegendentry{adaptive \eps = 0.2}
%\addplot [color=orange,densely dotted] table [x=adap_error_rate_0.3, y=adap_avg_cost_0.3] {chart1b_EM.data};
%\addlegendentry{adaptive \eps = 0.3}
\addplot [color=orange,dashdotted] table [x=adap_error_rate_0.5, y=adap_avg_cost_0.5] {chart1b_EM.data};
\addlegendentry{adaptive \eps = 0.5}
\addplot [color=black,loosely dotted,thick] table [x=adap_error_rate_1, y=adap_avg_cost_1] {chart1b_EM.data};
\addlegendentry{adaptive \eps = 1}
\end{axis}
\end{tikzpicture}
\caption{Cost-quality trade-off for fixed overlap majority, Get Another Label (GAL) and the adaptive algorithm on the adult data used
by Get Another Label.}
\label{fig:chart1b_EM}
\end{figure}

Finally, we repeat the same experiment with the adult data set from \cite{Ipeirotis10}. This data set contains classifications of web pages into four categories (G, PG, R, X), depending on the adult content on the page. There are 500 web pages with approximately 100 labels per page. We were unable to obtain from the authors the original gold labels of the web pages that they used for their experiments, therefore we computed the majority answer per web page and treated that as gold.
%All experiments were repeated 100 times.
Using the same adult data set, we also run the EM-based algorithm from \cite{Ipeirotis10}, known as Get Another Label (GAL). The original algorithm also contains a majority voting step that can be used to break ties using label priors (we call GAL with this step GAL majority vote).
For GAL, GAL majority vote  and the fixed overlap majority algorithms we randomly sampled a fixed number of worker labels per HIT (varying the fixed number to produce the curves), whereas for the other algorithms we randomly permuted the order of the worker labels.
The results we report are the averages of the error rates and costs over  100 runs  (Figure~\ref{fig:chart1b_EM}). We have omitted some curves for the adaptive algorithm for clarity. The performance of the adaptive algorithm is better than fixed overlap majority and GAL. This is not too surprising, considering that both GAL and fixed overlap majority use a fixed number of labels for every HIT
so they cannot really decrease the cost by stopping early like adaptive does.
The GAL majority vote algorithm performs comparably to the fixed overlap majority which is expected as they only differ on how they break ties.
Somewhat surprisingly, GAL performed worse than the GAL majority vote. We believe this was caused by the way we generated the ground truth. Since we could not obtain the original ground truth of the adult data set, we used the majority answer as ground truth and this may have impacted the results.

\vspace{-10mm}

%\begin{figure*}[t]
%        \centering
%        \begin{subfigure}[b]{0.5\textwidth}
%                \centering
%               \input{chart1.tex}
%%                \caption{Batch 1: 128 microtasks, 2 options each}
%                \label{fig:CDF-gap-a}
%        \end{subfigure}%
%        %$\qquad$
%          %(or a blank line to force the subfigure onto a new line)
%        \begin{subfigure}[b]{0.5\textwidth}
%                \centering
%                \input{chart2.tex}
% %               \caption{Batch 2: 604 microtasks, variable \#options}
%                \label{fig:CDF-gap-b}
%        \end{subfigure}
%    \caption{Cost vs. error for various $\eps$ values.}
%\end{figure*}

%\begin{figure}[h]
%\centering
%\input{chart1.tex}
%\caption{Cost-quality trade-off for various $\eps$ values.}
%\label{fig:chart1}
%\end{figure}
%
%\begin{figure}[h]
%\centering
%\input{chart2.tex}
%\caption{Cost-quality trade-off: zooming in on the most promising $\eps$ values.}
%\label{fig:chart2}
%\end{figure}

\vspace{1cm}

%\section{A Stopping rule with access to quality scores}
\section{Stopping rule for \\non-anonymous workers}
\label{sec:weighted}

Depending of the task and qualifications required, some workers may be better than others, and one can often estimate who is better by looking at the past performance.
We assume workers have a one-dimensional personal measure of {\em expertise} or skill level, which influences their error rate on HITs. Further, we assume we have access to a {\em reputation system} which can (approximately and coarsely) rank workers by their expertise level. We develop a weighted version of the stopping rule from Section~\ref{sec:unweighted} that is geared to take advantage of such a reputation system. We begin by describing a general weighted stopping rule, then detail how we use it.
\\
%\subsection{Weighted stopping rules}

\subsection{Algorithm}
%\xhdr{Our algorithm.}
In each round $t$, the worker is assigned weight $w_t$. In general, the weights may depend on the available information about the worker and the task. Also, the stopping rule can update the next worker's weight depending on the number $t$ itself. For now, we do not specify \emph{how} the weights are assigned. Absent any prior information on the workers, all weights are 1. Such stopping rules will be called \emph{unweighted}; we have discussed them in Section~\ref{sec:unweighted}.

\OMIT{A particularly simple class \emph{weighted} stopping rule is ones with $w_t\in \{0,1\}$; in other words, each worker $t$ is either given full weight ($w_t=1$), or completely ignored ($w_t=0$); we call it a \emph{0-1 weight} stopping rule.}

Fix some round $t$. The weighted vote $V_{A,t}$ for a given answer $A$ is defined as the total weight of all workers that arrived up to (and including) round $t$ and chose answer $A$. For simplicity, assume there are only two answers: $A$ and $B$. Our stopping rule is as follows:
\begin{align}\label{eq:stopping-rule-weighted}
\text{Stop if}\; \textstyle
|V_{A,t} - V_{B,t}| \geq C\,\sqrt{\sum_{s=1}^t w_s^2} - \eps \sum_{s=1}^t w_s.
\end{align}
Here $C>0$ and $\eps\in[0,1)$ are parameters that need to be chosen in advance. Note that in the unweighted case ($w_t\equiv 1$), this reduces to \refeq{eq:stopping-rule-unweighted}.
Our default selection rule is to choose the answer with the largest weighted vote. We call this the \emph{deterministic} selection rule.

%\ascomment{Do we also want to define some randomized selection rule? how?}

%\subsubsection{Discussion}
\xhdr{Discussion.}
The goal for weighted stopping rule is identical to the unweighted case: among the two hypotheses (H1) and (H2), reject the one that is wrong.

Letting
    $W_{t,q} = (\sum_{s=1}^t w_s^q)^{1/q}$,
we can re-write the stopping rule \eqref{eq:stopping-rule-weighted} more compactly as
\begin{align}\label{eq:stopping-rule-weighted-W}
\text{Stop if}\;
|V_{A,t} - V_{B,t}| \geq C\,W_{t,2} - \eps W_{t,1}.
\end{align}
The meaning of the right-hand side is as follows. $Z_t = V_{A,t} - V_{B,t}$ can be viewed as a biased random walk: its increments $Z_t-Z_{t-1}$ are independent random variables with values $\pm w_t$ and mean $\eps_0=\Pr[A]-\Pr[B]$. The expected drift of this random walk is 
$\E[Z_t] = \eps_0\, W_{t,1}$. Thus, the term $\eps W_{t,1}$ in \eqref{eq:stopping-rule-weighted-W} is a lower bound on the expected drift 
assuming either (H1) or (H2) holds. The meaning of the $C\,W_{t,2}$ term in \eqref{eq:stopping-rule-weighted-W} is that $W_{t,2}$ is the best available upper bound on the standard deviation of $Z_t$.

%\xhdr{Informal analysis.}
%Let's focus on the unweighted case for simplicity. Suppose the stopping rule fires at some time $t$ so that the answer $A$ with the maximal weighted vote is not the correct answer $A^*$. Then
%    $V_{A,t}-V_{A^*,t}> C\sqrt{t} - \eps t$.
%By the Azuma-Hoeffding inequality, this is very low probability event if the bias of each worker is at least $\eps$. More precisely, we have some option $A$ such that for $Z=V_{A,t}-V_{A^*,t}$ it holds that
%    $Z> \E[Z] + C \,\sigma(Z)$;
%the probability of that is at most $e^{-\Omega(C)}$. Furthermore, the threshold at right-hand side of \refeq{eq:stopping-rule-unweighted} is essentially the smallest possible threshold which enables the above argument.
%
%For the weighted case, the same argument applies. Here it is essential that
%    $\E[Z] \geq \sum_{s\leq t} w_s$,
%and
%    $\sigma(Z) \leq \sqrt{\sum_{s=1}^t w_s^2}$.
%
%\xhdr{Discussion.} Parameter $\eps$ represents the amount of ``slack" that we are willing to tolerate: essentially, we allow an incorrect answer if the workers' biases are less than $\eps$. Parameter $C$ controls the trade-off between the expected cost and the error rate: increasing $C$ increases the expected cost and decreases the error rate.

%\subsubsection{Extension to multiple answers}
\xhdr{Extension to multiple answers.}
It is easy to extend the stopping rule \eqref{eq:stopping-rule-weighted} to more than two answers. Let $A$ and $B$ be the answers with the largest and second-largest weighted vote, respectively. The stopping rule is
\begin{align}\label{eq:stopping-rule-weighted-multiple}
\text{Stop if}\;
V_{A,t} - V_{B,t} \geq C\,W_{t,2} - \eps W_{t,1}.
\end{align}

\newcommand{\rep}{\ensuremath{\mathtt{qty}}}
\newcommand{\good}{\ensuremath{\mathtt{good}}}
\newcommand{\average}{\ensuremath{\mathtt{average}}}
\newcommand{\bad}{\ensuremath{\mathtt{bad}}}

%\vspace{1cm}

%\subsubsection{Defining the weights}
\xhdr{Defining the weights.}
We restrict our attention to \emph{coarse} quality scores. This is because a reputation system is likely to be imprecise in practice, especially in relation to a specific HIT. So more fine-grained quality scores, especially continuous ones, are not likely to be meaningful.

Suppose each worker is assigned a coarse quality score $\rep$, e.g.
    $\rep \in \{ \good, \average, \bad \}$.
Our general approach, which we call \emph{reputation-dependent exponentiation}, is as follows. For each possible quality score $\rep$ we have an initial weight $\lambda_\rep$ and the multiplier $\gamma_{\rep}$. If in round $t$ a worker with quality score $\rep$ is asked, then her weight is
    $$w_t = \lambda_\rep\, \gamma_{\rep}^{t-1}$$
A notable special case is \emph{time-invariant weights}: $\gamma_\rep=1$.

The intuition is that we want to gradually increase the weight of the better workers, and gradually decrease the weight of the worse workers. The gradual increase/decrease may be desirable because of the following heuristic argument. As we found empirically (see Table~\ref{tab:worker-HIT-groups-diff}), the difference in performance between good and bad workers is more significant for hard HITs, whereas for very easy HITs all workers tend to perform equally well. Therefore we want to make the difference in \emph{weights} between the good and bad workers to be more significant for harder HITs. While we do not know a priori how difficult a given HIT is, we can estimate its difficulty as we get more answers. One very easy estimate is the number of answers so far: if we asked many workers and still did not stop, this indicates that the HIT is probably hard. Thus, we increase/decrease weights gradually over time.

\subsection{Experimental Results}

%\subsubsection{Simulated workload}

%To make the testing of various algorithms simpler and more streamlined, we first generate a simulated workload form the real workload.

\xhdr{Simulated workload.}
We use the real data 9-by-9 table of error rates for different worker and HIT groups to generate a simulated workload that
consists of 100,000 HITs, all with two answers, and 100 workers that answer all these HITs. We split workers uniformly across worker groups, and split HITs uniformly among HIT groups. For each worker and each HIT, the correct answer is chosen with the probability given by the corresponding cell in the table.
We define a coarse quality score depending on the worker group: the best three worker groups were designated \emph{good}, the middle three \emph{average} and the last three \emph{bad}. This quality score is given as input to the algorithm.%
%\footnote{The algorithm is not given the HIT group, because we believe that in practice the difficulty of a given HIT is essentially not known in advance, whereas the skill level of the workers can often be (coarsely) estimated from their previous performance.}

%\subsubsection{Algorithms tested}
\xhdr{Algorithms tested.}
We tested several ``reputation-dependent exponentiation'' algorithms. Recall that the weights in each such algorithm are defined by the initial weights $\lambda_\rep$ and the multipliers $\gamma_\rep$ for each quality score
        $\rep \in \{ \good, \average, \bad \}$.
For convenience, we denote the initial weights
    $\vec{\lambda} = (\lambda_\good,\, \lambda_\average,\, \lambda_\bad )$
and likewise the multipliers
    $\vec{\gamma} = (\gamma_\good,\, \gamma_\average,\, \gamma_\bad )$.
We experimented with many assignments for $(\vec{\lambda}, \vec{\gamma})$. Below we report on several paradigmatic versions:
%\vspace{-2mm}
\begin{enumerate}
\item No weights (all weights are set to 1, $\vec{\lambda} = (1, 1, 1)$):
    \begin{OneLiners}
    \item \emph{Fixed overlap majority:}  Same as section \ref{expResultsAnonymous}, uses a fixed number of annotations (i.e. overlap) per HIT and computes the simple majority answer. Ties are broken randomly. We vary the overlap parameter to produce the cost-error curve.
    \item \emph{Non-weighted adaptive:} The adaptive algorithm from the previous section, which assumes all the workers are anonymous.
%    \item[(V2)] bad workers ignored: $\vec{\lambda} = (1.2, 1, 0)$.
    \end{OneLiners}
\item Time-invariant weights (multipliers are set to 1, $\vec{\gamma} = (1, 1, 1)$):
    \begin{OneLiners}
\vspace{-2mm}
    \item \emph{Weighted fixed overlap:} Same as the fixed overlap majority algorithm, but the majority is weighted, using weights $\vec{\lambda} = (1.2, 1, 0.8)$.
    \item \emph{Fixed weights adaptive:} Similar to the adaptive algorithm, but uses the weighted stopping rule with weights $\vec{\lambda} = (1.2, 1, 0.8)$. The weights do not change during execution.
%    \item[(V2)] bad workers ignored: $\vec{\lambda} = (1.2, 1, 0)$.
    \end{OneLiners}
\item Time-varying weights. All weights start equal to 1 ($\vec{\lambda} = (1, 1, 1)$), but they increase/decrease by certain multipliers per round. We consider different multipliers $\vec{\gamma}$, so that weights change as follows.
    \begin{OneLiners}
    \item \emph{Reweighted adaptive 5\%:} The weights change by 5\% per round with multipliers $\vec{\gamma} = (1.05, 1, 0.95)$.
    \item \emph{Reweighted adaptive 10\%:} The weights change by 10\% per round with multipliers $\vec{\gamma} = (1.1, 1, 0.9)$.
    \item \emph{Reweighted adaptive 20\%:} The weights change by 20\% per round with multipliers $\vec{\gamma} = (1.2, 1, 0.8)$.
    \end{OneLiners}
%\oacomment{Can we expand how we came up with such weights?}
%\ascomment{Other than what we wrote, the weights are completely arbitrary. "paradigmatic versions", as we call them. :)}
%\item Time-varying weights, ignore bad workers:
%    \begin{OneLiners}
%    \item[(V7)] $\vec{\lambda} = (1, 1, 0)$, multipliers $\vec{\gamma} = (1.1, 1, 0)$.
 %   \end{OneLiners}
\end{enumerate}

%\begin{itemize}
%\item Time-invariant weights ($\gamma_\rep\equiv 1$). We consider different initial weights $\vec{\lambda}$:
%    \begin{OneLiners}
%    \item[(V0)] $\vec{\lambda} = (1, 1, 1)$.
%    \item[(V1)] $\vec{\lambda} = (1.2, 1, 0.8)$.
%    \item[(V2)] bad workers ignored: $\vec{\lambda} = (1.2, 1, 0)$.
%    \end{OneLiners}
%\item Time-varying weights, equal start ($\lambda_\rep \equiv 1$). We consider different multipliers $\vec{\gamma}$, so that weights change ...
%    \begin{OneLiners}
%    \item[(V3)] ... for all workers: $\vec{\gamma} = (1.05, 1, 0.95)$.
%    \item[(V4)] ... for all workers: $\vec{\gamma} = (1.1, 1, 0.9)$.
%    \item[(V5)] ... only for good workers: $\vec{\gamma} = (1.1, 1, 1)$.
%    \item[(V6)] ... only for bad workers: $\vec{\gamma} = (1, 1, 0.9)$.
%    \end{OneLiners}
%\item Time-varying weights, ignore bad workers:
%    \begin{OneLiners}
%    \item[(V7)] $\vec{\lambda} = (1, 1, 0)$, multipliers $\vec{\gamma} = (1.1, 1, 0)$.
%    \end{OneLiners}
%\end{itemize}

\noindent For each assignment of
    $(\vec{\lambda}, \vec{\gamma})$,
we consider several values for the parameter $\eps$, and for each $\eps$ we plotted a \curve in the error rate vs. expected cost plane.
To showcase our findings, some representative choices are shown in Figure~\ref{fig:chart3b}.

%\begin{figure}[t]
%\centering
%\input{chart3.tex}
%\caption{Cost-error trade-off for weighted stopping rules.
%    (For all \curves, $\eps = 0.3$.)}
%\label{fig:chart3}
%\end{figure}

\begin{figure}[t]
\centering
\begin{tikzpicture}
\begin{axis}[
        width=8.7cm,
        height=11cm,
        xlabel=Error rate,
        ylabel=Average cost,
        every axis y label/.style={at={(ticklabel cs:0.5)}, rotate=90, anchor=center, yshift=1.5mm},
        legend pos=north east,
        legend cell align=left,
        xmin=0,
        ymin=0,
        xticklabel style={
                /pgf/number format/.cd,
                      fixed,
                fixed zerofill,
                precision=2,
        },
        cycle list name=color list
]
%\addlegendimage{empty legend}
%\addlegendentry{\eps}
\addplot [color=olive,dotted] table [x=fixed_overlap_errorRate, y=fixed_overlap_avgCost] {chart3b.data};
\addlegendentry{fixed overlap majority}
\addplot [color=red,solid] table [x=weighted_fixed_overlap_errorRate, y=weighted_fixed_overlap_avgCost] {chart3b.data};
\addlegendentry{weighted fixed overlap}
\addplot [color=blue,dashed] table [x=adaptive_nonweighted_errorRate, y=adaptive_nonweighted_avgCost] {chart3b.data};
\addlegendentry{non-weighted adaptive}
\addplot [color=gray,dotted,thick] table [x=adaptive_fixedweights_errorRate, y=adaptive_fixedweights_avgCost] {chart3b.data};
\addlegendentry{fixed weights adaptive}
\addplot [color=orange,densely dotted] table [x=reweighted5pc_errorRate, y=reweighted5pc_avgCost] {chart3b.data};
\addlegendentry{reweighted 5\%}
\addplot [color=magenta,dashdotted] table [x=reweighted10pc_errorRate, y=reweighted10pc_avgCost] {chart3b.data};
\addlegendentry{reweighted 10\%}
\addplot [color=black,loosely dotted,thick] table [x=reweighted20pc_errorRate, y=reweighted20pc_avgCost] {chart3b.data};
\addlegendentry{reweighted 20\%}
\end{axis}
\end{tikzpicture}
\caption{Cost-error trade-off for weighted stopping rules.
    (For all \curves, $\eps = 0.2$.)}
\label{fig:chart3b}
\end{figure}
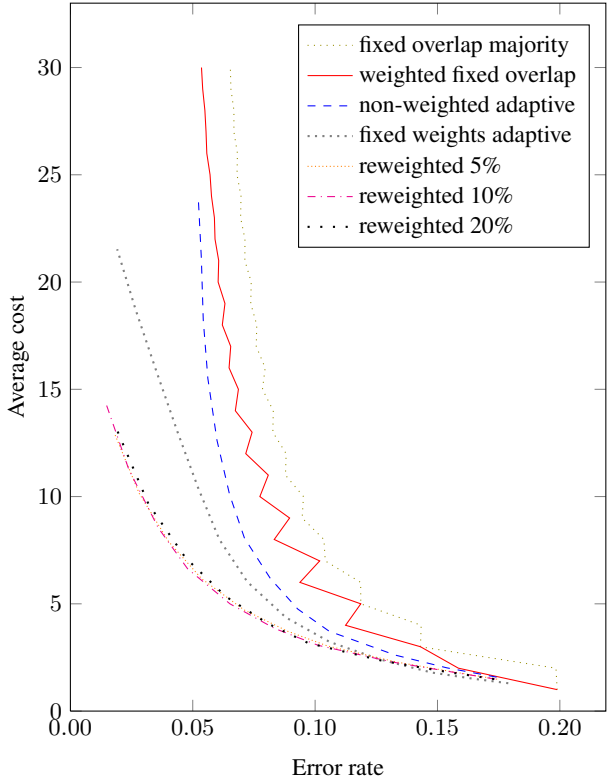

%\subsubsection{Findings}
\xhdr{Our findings.}
As in the previous section, we find that (up to some minor noise) for any two \curves, one lies below another. This enables comparisons between different algorithms that are valid for all choices of parameter $C$.  We conclude: %the following:

\begin{noindentlist2}
\item Using weights is better than not. The weighted fixed overlap algorithm performs better than the non-weighted version. Similarly, the adaptive algorithm performs better with weights (fixed or time-varying) than in the non-weighted (anonymous workers) case.
\item The adaptive algorithm performs better than the fixed overlap majority, even if the adaptive does not use weights, and the fixed overlap does.
\item For the adaptive, it is better to update the weights per round, rather than keeping them fixed.
\item How much the weights get updated does not have a big effect in performance (for the range of 5\% to 20\% that was tested).
%\oacomment{Is over time here in general or should we say that updating every round works best (as we say in the next paragraph?} \ascomment{"in general", at this point.}
%\item It is better to update the weights by a factor a larger factor, rather than a smaller one.
\item It is better to use $\eps>0$. The best value for $\eps$ is usually in the range $[.2, .3]$, and the effect of changing $\eps$ within this range is usually very small. This follows from experiments not shown in figure \ref{fig:chart3b}, but also supported in the experiments of the previous section and Figure~\ref{fig:chart1b}.
\end{noindentlist2}

We also experimented with various other combinations of weight updating schemes and multiplier values, which are not shown in the figure. We tried updating the weights every four rounds rather than every round (updating by $\gamma^4_\rep$, accordingly, for each quality score $\rep$), and we found that updating every round performs better. We also tried updating only the weights of the good workers (or only the weights of the bad workers) and the differences were very small.

Further, we investigated the effect of the magnitude of the multipliers $\vec{\gamma}$. We tried the previously mentioned weighted adaptive algorithm with multipliers that modify the worker weights by 5\%, 10\%, 20\%, 30\%, 40\% and 50\% (for example, $\vec{\gamma} = (1.3, 1, 0.7)$ for 30\% weight updates). We found the differences to be very small, with the updates of 5\% and 10\% to be very slightly better.
%\begin{OneLiners}
%\item $\vec{\gamma} = (1+\delta,1,1-\delta)$, where $\delta = 5\ldots 30\%$.
%\item $\vec{\gamma} = (1+\delta,1,1)$, where $\delta = 10\ldots 30\%$.
%\item $\vec{\gamma} = (1,1,1-\delta)$, where $\delta = 10\ldots 30\%$.
%\end{OneLiners}
%We found that the effect of changing the $\delta$ in this range is very insignificant.
%\oacomment{Can we present/format the above a bit better? If the argument is that the 30\% is stable, then let's just express it better}.
%\ascomment{Not sure, please optimize for clarity ...}

\section{Scalable Gold HIT creation}
\label{gold-hit}

\newcommand{\Index}{\ensuremath{\mathtt{Index}}}

We turn to the problem of scalable gold HIT creation, as described in the Introduction. We consider a stylized model with heterogeneity in worker quality but not in HIT difficulty. The system processes a stream of HITs, possibly in parallel. Each HIT is assigned to workers, sequentially and adaptively, at unit cost per worker, until the gold HIT answer is generated with sufficient confidence or the system gives up. Worker skill levels are initially not known to the algorithm, but can be estimated over time based on past performance. The goal is to minimize the total cost while ensuring low error rate.

% A.S.: no need to assume it here.
% For simplicity, let's assume that each HIT requires only binary answers.

%\subsection{Algorithm}
\xhdr{Algorithm.}
We adopt the following idea from prior work on multi-armed bandits: for each worker, combine exploration and exploitation in a single numerical score, updated over time, and at each decision point choose a worker with the highest current score~\cite{Thompson-1933,Gittins-index-79,bandits-ucb1}. This score, traditionally called an \emph{index}, takes into account both the average skill observed so far (to promote exploitation) and the uncertainty from insufficient sampling (to promote exploration). Over time, the algorithm zooms in on more skilled workers.

%In a typical application workers are given a blended stream of new HITs and Golds HITs that already belong to the Gold HIT set (so workers typically do not know of the current HIT is already part of the Gold HIT set).
% After obtaining enough workers to answer a HIT, we update the her index according to its .

We use a simple algorithm which builds on~\cite{bandits-ucb1,BanditSurveys-colt13}. For each worker $i$, let $t_i$ be the number of performed HITs for which the algorithm has generated a gold HIT answer, and let $t^+_i$ be the number of those HITs where the worker's answer coincides with the gold HIT. If $t_i\geq 1$, we define this worker's index as
\begin{align*}
 \Index_i = \frac{t^+_i}{t_i} + \frac{1}{\sqrt{t_i}}.
\end{align*}

Note that $\Index_i \leq 2$. For initialization, we set $\Index_i=2$.

Now that we've defined $\Index_i$, the algorithm is very simple:
\begin{noindentlist2}
\item At each time step, pick a worker with the highest index, breaking ties arbitrarily.
\item For each HIT, use the unweighted stopping rule \eqref{eq:stopping-rule-unweighted-multiple} to decide whether to stop processing this HIT. Then the gold HIT answer is defined as the majority answer.
\end{noindentlist2}

%\noindent If an exogenous reputation system is available, one can use the weighted stopping rules developed in Section~\ref{sec:weighted}.
%\oacomment{do we need this?}

\newcommand{\Dsucc}{\ensuremath{\mathcal{D}_{\mathtt{qty}}}}

%\subsection{Experimental Results}
\xhdr{Experimental setup.}
To study the empirical performance of our index-based algorithm, we use a simulation parameterized by real data as follows. We focus on HITs with binary answers. We have $1,000$ workers and each worker generates a correct answer for each HIT independently, with some fixed probability (\emph{success rate}) which reflects her skill level. The success rate of each worker is drawn independently from a realistic ``quality distribution'' \Dsucc.

We determined $\Dsucc$ by examining a large set ($>1,500$) of real workers from our internal platform (cf. Section \ref{data_analysis}), and computing their average success rates over
several months. Thus we obtained an empirical quality distribution, which we approximate by a low degree polynomial (see Figure~\ref{fig:quality}).

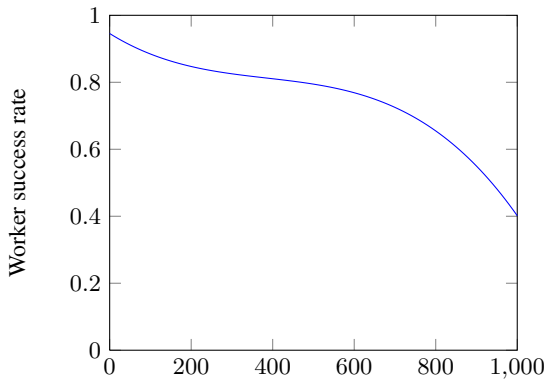
\begin{figure}[h]
\centering
\begin{tikzpicture}
\begin{axis}[
	width=7cm,
    %title={Worker quality},
    %xlabel=worker index,
    ylabel=Worker success rate,
    legend pos=north east,
		xmin=0,
		xmax=1000,
		ymin=0,
		ymax=1,
		%xticklabel style={
		%	/pgf/number format/.cd,
        %    fixed,
        %    fixed zerofill,
        %    precision=2,
		%},
]
\addplot [color=blue,mark=.] table [x=x, y=y] {chart_worker_pdf.data};
%\addlegendentry{probability being correct}
\end{axis}
\end{tikzpicture}
\caption{Worker quality distribution \Dsucc.}
\label{fig:quality}
\end{figure}

We compare our index-based algorithm to a naive algorithm, called \texttt{Random}, which assigns each HIT to a random worker. Both algorithms use the same unweighted stopping rule \eqref{eq:stopping-rule-unweighted-multiple}. In our simulation, each algorithm processes HITs one by one (but in practice the HITs could be processed in parallel).

Recall that the stopping rule comes with two parameters, $\eps$ and $C$. We consider three different values of $\eps$, namely $\eps=0$, $\eps=0.05$ and $\eps=0.1$. (Recall that according to our simulations in Section~\ref{sec:unweighted}, $[0.05,2]$ is the most promising range for $\eps$.) For each algorithm and each value of $\eps$, we vary the parameter $C$ to obtain different cost vs. quality trade-offs. For each value of $C$, we compute 5K gold HITs using each algorithm. Thus, for each algorithm and each value of $\eps$ we obtain a \curve.
The simulation results are summarized in Figure~\ref{fig:main-experiment}. The main finding is that our index-based algorithm reduces the per-HIT average cost by 35\% to 50\%, compared to \texttt{Random} with the same error rate. Recall that the cost here refers to the number of workers, which in practical terms translates to both time and money. Thus, we suggest adaptive exploration, and particularly index-based algorithms, as a very promising approach for automated gold HIT creation.

\begin{figure}[h]
\centering
\begin{tikzpicture}
\begin{axis}[
	width=8cm,
	height=10cm,
    xlabel=Error rate,
    ylabel=Average cost,
    legend pos=north east,
	legend cell align=left,
    xmin=0.0,
	xmax=0.1,
	ymin=0,
	xticklabel style={
		/pgf/number format/.cd,
		fixed,
		fixed zerofill,
		precision=2,
	},
]
\addplot [color=red,solid,mark=x] table [x=adap_error_eps_0, y=adap_cost_eps_0] {chart_5K_cost_vs_error.data};
\addlegendentry{Random (\eps = 0)}
\addplot [color=red,solid,mark=o] table [x=adap_error_eps_005, y=adap_cost_eps_005] {chart_5K_cost_vs_error.data};
\addlegendentry{Random (\eps = 0.05)}
\addplot [color=red,solid,mark=+] table [x=adap_error_eps_01, y=adap_cost_eps_01] {chart_5K_cost_vs_error.data};
\addlegendentry{Random (\eps = 0.1)}
\addplot [color=blue,mark=x,mark options={style=solid},dashed] table [x=index_error_eps_0, y=index_cost_eps_0] {chart_5K_cost_vs_error.data};
\addlegendentry{Index-based (\eps = 0)}
\addplot [color=blue,mark=o,mark options={style=solid},dashed] table [x=index_error_eps_005, y=index_cost_eps_005] {chart_5K_cost_vs_error.data};
\addlegendentry{Index-based (\eps = 0.05)}
\addplot [color=blue,mark=+,mark options={style=solid},dashed] table [x=index_error_eps_01, y=index_cost_eps_01] {chart_5K_cost_vs_error.data};
\addlegendentry{Index-based (\eps = 0.1)}
\end{axis}
\end{tikzpicture}
\caption{Simulation results.}
\label{fig:main-experiment}
\end{figure}
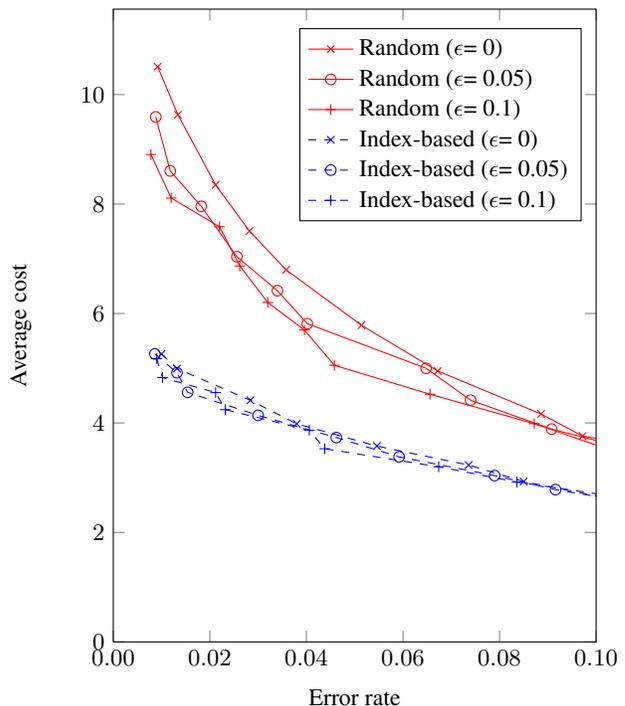

\section{Discussion and Conclusions}
\label{conclusions}

%Information retrieval relies on assessments and label quality for measuring performance quality. 
%As new and different types of media sources become available, the need emerges to build search and content organization infrastructure for them. 
%Gathering high quality labels  becomes an essential part of building search and content organization infrastructure. Ground truth generation and adjusting the number of workers in assessment tasks can thus cover a wide range of applications.

In this paper, we mainly focus on the issue of deciding how many workers to ask for a given HIT. The number of workers asked defines a trade-off between the cost of the HIT and the error rate of the final answer. We propose an adaptive stopping rule which, every time a worker is asked, decides whether to stop or continue asking another worker. The stopping rule takes into account the differences in the workers' answers and the uncertainty from the limited number of these answers. This allows asking few workers for easy HITs, where their answers are mostly identical, and thus incurring low cost. On the other hand, for harder HITs, more workers are asked in order to maintain a low error rate. A simpler scheme that uses a fixed number of workers per HIT wastes answers on the easy HITs and lacks enough answers on the harder HITs.

If workers' skill levels are approximately known from their past performance, we can improve the stopping rule to take the skill levels into account. The difficulty of a new HIT is, as before, assumed to be unknown. From our data analysis we know that all workers tend to perform well on easy HITs, whereas on harder HITs the skill level of the workers tends to make a big difference. We can thus estimate the HIT difficulty by the number of answers when the stopping rule decided to stop. We use this information to re-weight the answers of the workers according to their known skill, so that for harder HITs we rely more on the better workers.
With other EM-based algorithms, the assumption is that we do not know much about the workers and we try to assign a score to them that corresponds to how good they are (e.g., spammers get a low score, whereas workers that are giving correct answers get a high score). The EM-based algorithms try to estimate the worker scores at the same time as estimating the HIT answers. If workers give answers that agree with the estimated ones, then they tend to get high scores. Our adaptive algorithm instead assumes that we know how good or bad the workers are. 
%(e.g. from their prior performance).
 What it tries to estimate  is how easy or hard the HIT is, and what should be the HIT answer. The idea is that for easy HITs, even poorly performing workers are quite reliable
% (by the way, this is why we assume that any spammers have been removed; bad workers are not malicious and can be useful in some easy cases).
 So, if we can figure out that a HIT is easy we can rely on pretty much every worker, whereas if a HIT is hard we should discount the answers of the bad workers. The adaptive algorithm estimates the difficulty of the HIT based on how much the workers agree or disagree and then assigns a weight to each worker's answer that relies on both the estimated difficulty of the HIT and the worker quality.
% (which is known beforehand). 
This also means that the worker weights differ from HIT to HIT.

One can envision an approach where the worker skill is not known beforehand but can be learned algorithmically. For example, after the stopping rule decides to stop and produce a final answer for the HIT, we could compare the worker's answer to the final answer. If their answer matches, we can assume they gave a correct answer. This approach is particularly suitable to the problem of scalable gold HIT creation. However, further research is required to establish if this can produce accurate results in practice or if it leads to ``self-fulfilling loops'' where the workers who are considered skilled provide the same wrong answer. Such answer is then interpreted as the ``correct'' answer by the system, which in turn reinforces the belief that these workers are highly skilled.

While our stopping rules return a single answer for a given HIT, they can be extended to HITs with \emph{several} correct answers. For example, if the vote difference is small between the top two answers, but large between the second and the third answer, then we could stop and output the top two answers as both being correct. With similarly simple modifications, the rules can be expanded to deal with HITs in which the answers correspond to specific numerical values. In that case, it is not only the vote difference that matters but also the difference between the corresponding numerical values. These extensions are the subject of future research.

\xhdr{Acknowledgments.}
We thank Jennifer Wortman Vaughan for providing valuable feedback.

% ############ BIBLIOGRAPHY ##############

%\begin{small}
\bibliographystyle{plain}
\bibliography{bib-abbrv-short,bib-slivkins,bib-bandits,bib-crowdsourcing}
%\end{small}

\end{document}